\documentclass[a4paper,12pt]{article}

\usepackage[utf8]{inputenc}
\usepackage{amssymb}
\usepackage{amsmath}
\usepackage{graphicx}
\usepackage{tikz-cd}
\usepackage{amsthm}
\usepackage[colorlinks=true, allcolors=blue]{hyperref}
\usepackage[english]{babel}
\usepackage{enumitem}  
\usepackage{comment}
\usepackage{xspace}
\usepackage{algpseudocode}
\usepackage{algorithm2e}
\RestyleAlgo{ruled}

\newcommand{\oea}{\mbox{${(1 + 1)}$~EA}\xspace}

\newcommand{\oplea}{\mbox{${(1+\lambda)}$~EA}\xspace}

\newcommand{\oclea}{\mbox{${(1,\lambda)}$~EA}\xspace}

\newcommand{\MAHH}{MAHH\xspace}

\newcommand{\onemax}{\textsc{OneMax}\xspace}
\newcommand{\LO}{\textsc{Leading\-Ones}\xspace}
\newcommand{\leadingones}{\LO}

\newcommand{\cliff}{\textsc{Cliff}\xspace}

\newcommand{\jump}{\textsc{Jump}\xspace}

\newcommand{\N}{\ensuremath{\mathbb{N}}}

\newcommand{\R}{\mathbb{R}}

\newtheorem{theorem}{Theorem}
\newtheorem{lemma}[theorem]{Lemma}
\newtheorem{corollary}[theorem]{Corollary}

\newcommand{\improvingandequal}{\textsc{ImprovingAndEqual}\xspace}
\newcommand{\IMPROVINGANDEQUAL}{\improvingandequal}
\newcommand{\onlyimproving}{\textsc{OnlyImproving}\xspace}
\newcommand{\ONLYIMPROVING}{\onlyimproving}
\newcommand{\allmoves}{\textsc{AllMoves}\xspace}
\newcommand{\ALLMOVES}{\allmoves}
\newcommand{\acc}{\textsc{Acc}\xspace}
\newcommand{\ACC}{\acc}

\newcommand{\eps}{\varepsilon}

\let\originalleft\left
\let\originalright\right
\renewcommand{\left}{\mathopen{}\mathclose\bgroup\originalleft}
\renewcommand{\right}{\aftergroup\egroup\originalright}

\usepackage{hyperref}

\newcommand{\eqnref}[1]{(\ref{#1})}

\DeclareMathOperator{\expectation}{E}

\begin{document}
{\sloppy

\title{Hyper-Heuristics Can Profit From Global Variation Operators
\thanks{Continues work presented in the conference paper \emph{Benjamin Doerr, Arthur Dremaux, Johannes Lutzeyer, and
Aurélien Stumpf. How the move acceptance hyper-heuristic
copes with local optima: drastic differences between jumps and
cliffs. In Genetic and Evolutionary Computation Conference,
GECCO 2023, pages 990–999. ACM, 2023}~\cite{DoerrDLS23}. Thanks to Arthur and Aurélien for their kind permission to include parts of this joint work here (namely the lower bound proof for the MAHH with local mutation). The present work contains new upper bound proofs for both the case of local and global mutation. The latter is the first proof showing that hyperheuristics can gain asymptotic speed-ups from global mutation operators.}}

\author{Benjamin Doerr 
\and Johannes F. Lutzeyer 
}

\maketitle
\begin{abstract}
In recent work, Lissovoi, Oliveto, and Warwicker (Artificial Intelligence (2023)) proved that the Move Acceptance Hyper-Heuristic (\MAHH) leaves the local optimum of the multimodal \cliff benchmark with remarkable efficiency. The $O(n^3)$ runtime of the \MAHH, for almost all cliff widths $d\ge 2,$ is significantly better than the $\Theta(n^d)$ runtime of simple elitist evolutionary algorithms (EAs) on \cliff. 
  
In this work, we first show that this advantage is specific to the \cliff problem and does not extend to the \jump benchmark, the most prominent multi-modal benchmark in the theory of randomized search heuristics. We prove that for any choice of the \MAHH selection parameter~$p$, the expected runtime of the \MAHH on a \jump function with gap size $m = O(n^{1/2})$ is at least $\Omega(n^{2m-1} / (2m-1)!)$. This is significantly slower than the $O(n^m)$ runtime of simple elitist EAs. 

Encouragingly, we also show that replacing the local one-bit mutation operator in the \MAHH with the global bit-wise mutation operator, commonly used in EAs, yields a runtime of $\min\{1, O(\frac{e\ln(n)}{m})^m\} \, O(n^m)$ on \jump functions. This is at least as good as the runtime of simple elitist EAs. For larger values of $m$, this result proves an asymptotic performance gain over simple EAs. As our proofs reveal, the MAHH profits from its ability to walk through the valley of lower objective values in moderate-size steps, always accepting inferior solutions. This is the first time that such an optimization behavior is proven via mathematical means. Generally, our result shows that combining two ways of coping with local optima, global mutation and accepting inferior solutions, can lead to considerable performance gains. 
\end{abstract}

\section{Introduction}

Evolutionary algorithms (EAs) have undergone intense mathematical analysis over the last 30 years~\cite{NeumannW10,AugerD11,Jansen13,ZhouYQ19,DoerrN20}. Contrariwise, the rigorous analysis of hyper-heuristics is far less developed and has, so far, mostly focused on unimodal problems, which has led to several highly interesting results (see Section~\ref{sec:prev} for further detail).

In the first, and so far only, work on how hyper-heuristics solve multimodal problems, that is, how they cope with the presence of true local optima, Lissovoi, Oliveto, and Warwicker have proven the remarkable result that the simple Move Acceptance Hyper-Heuristic (\MAHH), which randomly mixes elitist selection with accepting any new solution, optimizes all \cliff functions in time $O(n^3)$ (or less depending on the cliff width~$d$). The \MAHH thus struggles much less with the local optimum of this benchmark than several existing approaches, including simple elitist EAs (typically having a runtime of $\Theta(n^d)$~\cite{PaixaoHST17}), the \oclea having a runtime of roughly $O(n^{3.98})$ (shown for $d=n/3$ only, which required a very careful choice of the population size)~\cite{FajardoS21foga}, and the Metropolis algorithm (with runtime $\Omega(n^{d-0.5} / (\log n)^{d-1.5})$ for constant $d$ and a super-polynomial runtime for super-constant~$d$).

The surprisingly good performance of the \MAHH on \cliff functions raises the question: Does the convincing performance of the \MAHH on \cliff generalize to other functions or is it restricted to the \cliff benchmark only? To answer this question, we study the performance of the \MAHH on the multimodal benchmark most prominent in the mathematical runtime analysis of randomized search heuristics, the \jump benchmark. For this problem, with jump size $m\ge 2$, only the loose bounds 
$O(n^{2m} m^{-\Theta(m)})$ and $\Omega(2^{\Omega(m)})$
were shown in~\cite{LissovoiOW23}. These bounds allow no conclusive comparison with simple EAs, which typically have a  $\Theta(n^m)$ runtime on \jump functions. 

In this work, we prove a general non-asymptotic lower bound for the runtime of the \MAHH on \jump functions, valid for all values of the problem size~$n$, jump size~$m \ge 2$, and mixing parameter~$p$ of the hyper-heuristic. As the most interesting special case of this bound, we derive that for $m = O(n^{1/2})$ and all values of $p$, this runtime is at least $\Omega(\frac{n^{2m-1}}{(2m-1)!})$, which is significantly larger than the $\Theta(n^m)$ runtime of many EAs. Consequently, our bound allows us to answer the question on the generalization of the good performance of the \MAHH shown in~\cite{LissovoiOW23}: The \MAHH performs comparatively poorly on \jump functions, i.e., the surprisingly good performance on \cliff does not extend to \jump functions. 

We note that our lower bound is relatively tight: We show that the \MAHH with $p = \frac mn$ has a runtime of $O(\frac{n^{2m-1}}{m! m^{m-2}})$. 

The significant performance gap between, e.g., the \oea and the \MAHH, motivates us to propose to equip the \MAHH with bit-wise mutation, the variation operator of the \oea and the most common mutation operator in evolutionary computation, instead of the one-bit flips common in hyper-heuristics. This global mutation operator renders the analysis of the \MAHH significantly more complex and in particular forbids the use of the Markov chain arguments employed in the one-bit mutation case. With a suitable potential function and a drift argument, we manage to show an upper bound of $\min\{1, O(\frac{e\ln(n)}{m})^m\} \, O(n^m)$, a significant speed-up over the one-bit case,  and over simple evolutionary algorithms when $m$ is not too small. This upper bound shows that, in principle, combining the global variation operator popular in evolutionary computation with a local-search hyper-heuristic can be an interesting approach. The potential of such combinations was discussed for the Metropolis algorithm in~\cite{DoerrERW23}, where it is analyzed only via experiments on \cliff functions with problem size $n=100$ and cliff height $d=3$. Hence our work is the first time that such speed-ups are proven in a rigorous manner. 

We believe that this is a promising direction for future research. We note, however, that we were not able to extend the $O(n^3)$ upper bound for \cliff functions from the one-bit mutation MAHH to our new algorithm, so we cannot rule out that the strong performance of the classic \MAHH on \cliff functions is lost when using bit-wise mutation. More research in this direction to answer such questions is needed. At the moment, the main missing piece towards more progress is a method to prove lower bounds when less restricted mutation operators than one-bit flips are used.

The remainder of this paper is structured as follows. In Section~\ref{sec:prev} we review related work. Then, in Section~\ref{sec:preliminaries} we introduce our notation and relevant background knowledge. Subsequently, in Section~\ref{sec:lower_bound} we prove a lower bound on the runtime of the \MAHH with one-bit mutation on the \jump benchmark. Our main result is stated in Section~\ref{sec:GlobalMutation}, in which we prove an upper bound on the runtime of the \MAHH with global mutation on the \jump benchmark. In Section~\ref{sec:upper_bound}, we establish the relative tightness of our lower bound produced in Section~\ref{sec:lower_bound}. In Section~\ref{sec:conclusion}, we conclude this work with a short summary and discussion.

\section{Related Work}\label{sec:prev}

Since we perform a mathematical runtime analysis in this paper, we concentrate our review of related work on this domain. We refer to~\cite{BurkeGHKOOQ13} for a general introduction to hyper-heuristics and to~\cite{NeumannW10,AugerD11,Jansen13,ZhouYQ19,DoerrN20} for introductions to mathematical runtime analyses of randomized search heuristics.

\subsection{Hyper-Heuristics}

The mathematical runtime analysis of hyper-heuristics is a recent field that was started by Lehre and \"Oczan~\cite{LehreO13} by working on random mixing of mutation operators and random mixing of acceptance operators. Their work has led to numerous follow-up works, among which we now describe the most important ones. We refer to the survey~\cite[Section~8]{DoerrD20bookchapter} for more details.

All mathematical runtime analyses of hyper-heuristics work with \emph{selection heuristics}, that is, hyper-heuristics that try to choose wisely among a fixed set of given operators. The majority of these works consider the selection of mutation operators. The simplest selection hyper-heuristic is random mixing between operators, that is, selecting each time independently an operator from a given probability distribution. Random mixing of the 1-bit-flip and 2-bit-flip mutation operator was studied in~\cite{LehreO13}. Their runtime analysis on \onemax (see~\cite[Theorem~6.8.1]{DoerrD20bookchapter} for a corrected version of this result) indicated no advantages of mixing these two operators, a natural result for this unimodal benchmark. To show that mixing of mutation operators can be essential, the \textsc{GapPath} problem was constructed in~\cite{LehreO13}. Neumann and Wegener~\cite{NeumannW07} showed that the minimum spanning tree problem is such an example as well, without calling the studied algorithm a hyper-heuristic. 

We note that the result of randomly mixing mutation operators can simply be regarded as a new mutation operator. For example, bit-wise mutation can be seen as a suitable mixing of $k$-bit-flip operators ($k = 0, 1, \dots, n$) or the fast mutation operator of~\cite{DoerrLMN17} can be seen as a mixing of bit-wise operators with different mutation rates. 
Therefore, more complex ways of selecting between different operators appear to be more interesting. In terms of runtime analyses, this was first done in~\cite{AlanaziL14}, again to select between different mutation operators, however, no superiority over random mixing could be shown. 
This result was due to a misconfiguration of these hyper-heuristics as was shown in the seminal work of Lissovoi, Oliveto, and Warwicker~\cite{LissovoiOW20ecj}. In this work, the authors show that the three more complex hyper-heuristics optimize the \leadingones benchmark in exactly (ignoring lower order terms) the same runtime as the simple random mixing heuristics. 
They observe that the true reason for these equivalent runtimes is the short-sightedness of the considered heuristics: As soon as no improvement is found, they switch to a different low-level heuristic. If instead longer learning periods are used, that is, a change of the currently used low-level heuristic is made only after a certain number of successive failures, then the more complex hyper-heuristics perform very well and attain the best performance that can be obtained from these low-level heuristics. 
The length of the learning period is crucial for the success of the hyper-heuristic and can be determined in a self-adjusting manner as shown in~\cite{DoerrLOW18}. 
These works were extended to other benchmark problems in~\cite{LissovoiOW20aaai}.

Less understood are selection hyper-heuristics working with different acceptance operators. In~\cite{LehreO13}, a random mixing between an all-moves (\allmoves) operator, accepting any new solution, and an only-improving (\ONLYIMPROVING) operator, accepting only strict improvements, was proposed and it was shown that the (unimodal) royal-road problem can only be solved by mixing these two operators. As noted in~\cite{LissovoiOW23}, this result heavily relies on the fact that the \ONLYIMPROVING operator was used instead of the, in evolutionary computation more common, improving-and-equal (\IMPROVINGANDEQUAL) operator, which accepts strict improvements and equally good solutions. 

A huge step towards a mathematical understanding of hyper-heuristics was recently made by Lissovoi, Oliveto, and Warwicker~\cite{LissovoiOW23}. Their results are valid equally for the \ONLYIMPROVING and the \IMPROVINGANDEQUAL operator. For this reason, let us collectively refer to the two versions of the hyper-heuristic which randomly mix between the \allmoves (with probability~$p$, the only parameter of the algorithm) and either the \ONLYIMPROVING or \IMPROVINGANDEQUAL operator, as the  \emph{Move Acceptance Hyper-Heuristic (\MAHH)}. In other words, the \MAHH starts with a random solution and then, in each iteration, moves to a random neighbor of the current solution and, with probability~$p$  always accepts this move and with probability $1-p$ only accepts this move if the new solution is at least as good or strictly better. See Section~\ref{sec:preliminaries} for a more detailed description of the \MAHH algorithm. 

The striking result of~\cite{LissovoiOW23} is that the simple \MAHH with mixing parameter $p = \frac{1}{(1+\eps)n}$, with constant $\eps > 0$ and problem size $n$, optimizes the \cliff benchmark  with cliff width~$d$ in an expected number of $O(\frac{n^3}{d^2} + n \log n)$ iterations. This is remarkably efficient compared to the runtimes of other algorithms on this benchmark (see Section~\ref{sec:runtime_literature}). We note that~\cite{LissovoiOW23} is the first work that analyzes with mathematical means the performance of a hyper-heuristic on a classic multimodal benchmark problem. 

The extremely good performance of the MAHH on the \cliff benchmark shown in~\cite{LissovoiOW23} raises the question of whether the \MAHH has the general strength of leaving local optima or whether this behavior is particular to the \cliff problem, which with its particular structure (a single acceptance of an inferior solution suffices to leave the local optimum) appears to be the perfect benchmark for the \MAHH. 

To answer this question, the authors of~\cite{LissovoiOW23} also consider other multimodal problems, however, with non-conclusive results.  
For the most prominent multimodal benchmark $\jump$ with gap parameter $m \ge 2$, in~\cite[Theorems 12 and 13]{LissovoiOW23} the bounds $\Omega(2^{\Omega(m)} + n \log n)$ and $O(\frac{(1+\eps)^{m-1} n^{2m}}{m^2 m!})$ were shown for the parameter choice $p=\frac{1}{(1+\eps)n}$. We note that the conference version~\cite{LissovoiOW19} states, without proof, a stronger upper bound of $O(\frac{n^{2m-1}}{m})$, but since the journal version only proves a weaker bound, we assume that 
only the bound in the journal version is valid. These bounds are distant to the known runtime $\Theta(n^m)$ of simple evolutionary algorithms on the \jump benchmark, and thus do not allow a conclusive comparison of these algorithms. 

\subsection{Runtime Analyses on \cliff and \jump Functions}\label{sec:runtime_literature}

We now recall the main existing runtime results for the \cliff and \jump benchmarks. \cliff functions (with fixed cliff width $d =  n/3$) were introduced in~\cite{JagerskupperS07} to demonstrate that the \oclea can profit from its non-elitism. This result was further refined in~\cite{FajardoS21foga}. For the best choice of the population size, a runtime of approximately $O(n^{3.98})$ was shown, however, the result also indicates that small deviations from the optimal parameter choice lead to significant performance losses. 

Comparably simple elitist EAs can leave the local optimum of a general \cliff function, and similarly \jump function, with cliff width $d$ only by flipping a particular set of $d$ bits. Hence they have a runtime of $\Theta(n^d)$ when using bit-wise mutation with mutation rate $\frac 1n$ (this follows essentially from the result on \jump functions in the classic work~\cite{DrosteJW02}, but was shown separately for \cliff functions in~\cite{PaixaoHST17}). When using the heavy-tailed mutation operator proposed in~\cite{DoerrLMN17}, again as for \jump functions, the runtime reduces to $n^d d^{-\Theta(d)}$ for $d\geq 2$. A combination of mathematical and experimental evidence suggests that the compact genetic algorithm has an exponential runtime on the \cliff function with $d=n/3$~\cite{NeumannSW22}.

The Metropolis algorithm profits at most a little from its ability to leave local optima when optimizing \cliff functions. For constant~$d$, a lower bound of $\Omega(n^{d-0.5} (\log n)^{-d+1.5})$ was shown in~\cite{LissovoiOW23}, for super-constant $d$ it was shown that the runtime is super-polynomial. For a recent tightening and extension of this result, we refer to~\cite{DoerrERW23}.

Several artificial immune systems employing an aging operator were shown to optimize \cliff functions in time $O(n \log n)$~\cite{CorusOY20}. Since artificial immune systems are even less understood than hyper-heuristics, it is difficult to gauge the general meaning of this result. In summary, we agree with the authors of~\cite{LissovoiOW23} that the performance of the \MAHH on \cliff functions is remarkably good. 

The \jump functions benchmark is by far the most studied multimodal benchmark in the theory of randomized search heuristics. It was proposed already in~\cite{DrosteJW02}, where the runtime of the \oea on this benchmark was shown to be $\Theta(n^m)$ for all values of $m$. Since then the \jump benchmark has been intensively studied and given rise to many important results, e.g., \jump functions are one of the few examples where crossover was proven to result in significant performance gains~\cite{JansenW01,KotzingST11,DangFKKLOSS18,AntipovDK22}, estimation-of-distribution algorithms and ant-colony optimizers were shown to significantly outperform classic evolutionary algorithms on \jump functions~\cite{HasenohrlS18,Doerr21cgajump,BenbakiBD21}, and it led to the development of fast mutation~\cite{DoerrLMN17} and a powerful stagnation-detection mechanism~\cite{RajabiW22}. Several variations of jump functions have been proposed in~\cite{Jansen15, RajabiW21gecco, DoerrZ21aaai,Witt23, BamburyBD24, DoerrGI24}.

In the context of our work, it is important to note that typical elitist mutation-based algorithms optimize \jump functions in time $\Theta(n^m)$. A speed-up by a factor of $\Omega(m^{\Omega(m)})$ can be obtained from fast mutation~\cite{DoerrLMN17} and various forms of stagnation detection~\cite{RajabiW22,RajabiW23,RajabiW21gecco,DoerrR23} (where stagnation detection usually produces runtimes by a factor of around $\sqrt m$ smaller than fast mutation). With crossover, the $(\mu+1)$ genetic algorithm (GA) without additional modifications reaches runtimes of $\tilde O(n^{m-1})$~\cite{DangFKKLOSS18,DoerrEJK24}. With suitable diversity mechanisms or other additional techniques, runtimes up to $O(n)$ were obtained~\cite{DangFKKLOSS16,FriedrichKKNNS16,WhitleyVHM18,RoweA19,AntipovD20ppsn,AntipovDK22,AntipovBD24}, but the lower the runtimes become, the more these algorithms appear custom-tailored to jump functions, see, e.g.,~\cite{Witt23}. The extreme end is marked by an $O(\frac{n}{\log n})$ runtime algorithm~\cite{BuzdalovDK16} designed to exhibit the black-box complexity of jump functions.  

Non-elitism could not be used effectively on \jump functions so far. The \oclea for essentially all reasonable parameter choices was proven to have at least the runtime of the \oplea~\cite{Doerr22} (but note that comma selection was recently shown to be profitable on randomly disturbed \onemax functions~\cite{JorritsmaLS23}). In~\cite[Theorem~14]{LissovoiOW23}, the Metropolis algorithm for any value of the acceptance parameter was shown to have a runtime of at least $2^{\Omega(n)}$ with at least constant probability.\footnote{The result~\cite[Theorem~14]{LissovoiOW23} says ``with probability $1-2^{-\Omega(n)}$'', but this seems to overlook that the proof from the second sentence on only regards the case in which the random initial solution has at least $n/2$ zeroes.}

\section{Preliminaries}\label{sec:preliminaries}

We now formally define the \MAHH algorithm, our considered standard benchmark functions, as well as some mathematical tools used in our runtime analysis. Throughout, we use the notation $[1..n]$ to denote the set of integers $\{1, 2, \dots, n\}.$

\subsection{Algorithms}\label{sec:algorithms}

We will analyze the runtime of the \MAHH algorithm applied to the problem of reaching the optimum of a benchmark function defined on the space of $n$-dimensional bit vectors.

\begin{algorithm}[t]
\caption{Move Acceptance Hyper-Heuristic (\MAHH).}\label{alg:mahh}
\textbf{Input:} Choose $x \in \{0,1 \}^n$ uniformly at random\\
\While{$\mathrm{termination~criterion~not~satisfied}$}{
  $x^{\prime} \gets$ sample a bit of $x$ uniformly at random and flip it;
  $\ACC \gets 
  \begin{cases}
      \allmoves, & \text{with probability } p; \\
      \ONLYIMPROVING, & \text{otherwise};\\
  \end{cases}$\\
  \If{$\ACC(x,x^{\prime})$}{
    $x \gets x^{\prime}$\;
  }
}
\end{algorithm}

In each iteration of the algorithm, one bit of the current vector $x$ is chosen at random and flipped to create a mutation $x^{\prime}$. This mutation is accepted with probability $p$ (\ALLMOVES operator), and is accepted only if the value of the benchmark function increases, with probability $1-p$ (\onlyimproving operator). The pseudocode of the Move Acceptance Hyper-Heuristic can be found in Algorithm~\ref{alg:mahh}. We do not specify a termination criterion here, since our aim is to study how long the MAHH takes to reach a target when not stopped prematurely.

\subsection{Benchmark Function Classes}

We will now define the various benchmark functions that will be used to analyze the performances of the considered algorithms. 

For all $x\in\{0,1\}^n$, the $\onemax$ benchmark is defined by
\[
\onemax(x) = \|x\|_1 = \sum_{i=1}^{n}x_i.
\]
The $\onemax$ function has a constant slope leading to a global optimum placed at the all-ones bit-string $x^* = (1, \dots, 1)$. It is often used to evaluate the hillclimbing performance of randomized search heuristics. It is also the basis for the definition of the \cliff and \jump benchmarks.

While we will not study the \cliff benchmark in detail, for completeness we briefly define it now. The \cliff benchmark was first proposed in~\cite{JagerskupperS07}. For $ d \in [1..n]$, the function $\cliff_d$ is defined by 
\begin{eqnarray*}
\cliff_d(x) = 
\begin{cases}
      \onemax(x), & \text{if } \lVert x\rVert_1 \leq n-d; \\
      \onemax(x) - d + \tfrac{1}{2}, & \text{otherwise,}\\
\end{cases}  
\end{eqnarray*}
for all $x\in\{0,1\}^n$.
$\cliff$ functions consist of two \onemax-type slopes pointing towards the global optimum. The second one of these is constructed such that all its solutions except the global optimum have an objective value inferior to the best objective value in the first slope. Consequently, search heuristics that may accept inferior solutions may leave the local optimum by moving to an inferior solution on the second slope and improving this to the global optimum.

In this paper, we will present a runtime analysis of the \MAHH hyperheuristic on the $\jump$ benchmark. For $ m \in [1..n]$, the function $\jump_m$ is defined for all $x\in\{0,1\}^n$ via
\begin{eqnarray*}
\jump_m(x) = 
\begin{cases}
      m+\onemax(x), & \text{if } \lVert x\rVert_1 \leq n-m \text{ or } \|x\|_1 = n; \\
      n-\onemax(x), & \text{otherwise.}\\
\end{cases}  
\end{eqnarray*}
Clearly, the main difference between \cliff and \jump functions is the shape of the valley of lower objective values surrounding the global optimum. For \cliff functions, this is ascending towards the global optimum, whereas for \jump functions it is descending, that is, the gradient is pointing away from the global optimum.

\subsection{Mathematical Tools}

We now introduce several of the fundamental mathematical tools, that we use to prove our upper bounds in Sections~\ref{sec:GlobalMutation}~and~\ref{sec:upper_bound}. In particular, we will be recalling the multiplicative and additive drift theorems as well as a simplified version of Wald's equation.

We begin with the multiplicative drift theorem, which provides upper bounds on the expected runtime when the progress can be bounded from below by an expression proportional to the distance from the target (which is set to zero in the formulation of Theorem~\ref{thm:mult-drift}). 

\begin{theorem}[Multiplicative Drift Theorem \cite{DoerrJW12algo}]
\label{thm:mult-drift}
Let $S \subseteq \mathbb{R}$ be a finite set of positive numbers with minimum $s_{\min}$. Let $(X_t)_{t\geq0}$ be a sequence of random variables over $S \cup \{0\}$. Let $T$ be the random variable that denotes the first point in time $t \in \mathbb{N}$ for which $X_t = 0$. 
Suppose further that there exists a constant $\delta > 0$ such that 
\[
\expectation\left[X_{t} - X_{t+1} \mid  X_t = s \right] \geq \delta s
\]
holds for all $s \in S$ with $\Pr\left[X_t = s\right] > 0$. 
Then, for all $s_0 \in S$ with $\Pr\left[X_0 = s_0 \right] > 0$,
\[
\expectation\left[T \mid  X_0 = s_0 \right] \leq \frac{1 + \log(s_0/s_{\min})}{\delta}.
\]
\end{theorem}

We continue with the additive drift theorem of He and Yao~\cite{HeY01} (see also the recent survey~\cite{Lengler20bookchapter}), which allows to translate an expected progress (or bounds on it) into bounds for expected hitting times.   
  
\begin{theorem}[Additive Drift Theorem \cite{HeY01}]
\label{thm:adddrift}
  Let $S \subseteq \R_{\ge 0}$ be finite and $0 \in S$. Let $(X_t)_{t\geq0}$ be a sequence of random variables over $S$. Let $\delta > 0$. Let $T = \inf\{t \ge 0 \mid X_t = 0\}$.
If for all $t \ge 0$ and all $s \in S \setminus \{0\}$ we have $$\expectation[X_t - X_{t+1} \mid X_t = s] \ge \delta,$$ then $$\expectation[T] \le \frac{\expectation[X_0]}{\delta}.$$
\end{theorem}

We finally state a simplified version of Wald's equation from~\cite{DoerrK15}. Wald's equation~\cite{Wald44} allows for computing or estimating the expectation of a sum of random variables, where the range of the summation is also controlled by a random variable (a so-called stopping time). The following version circumvents the use of stopping times.

\begin{theorem}[Simplified Version of Wald's Equation \cite{Wald44,DoerrK15}]
\label{thm:wald}
Let $T$ be a random variable with bounded expectation and let $(X_t)_{t\geq0}$ be a sequence of non-negative random variables with $\expectation[X_i\mid T \ge i ]\le C$. Then
    \[ \expectation\left[\sum_{i=1}^T X_i \right] \le C \expectation[T]. \]
\end{theorem}

\section{Lower Bounds on the Runtime of the \MAHH with One-Bit Mutation}\label{sec:lower_bound}

We now begin with the mathematical runtime analysis of the \MAHH.  
We first calculate an exact formula for the expected runtime of the \MAHH on the \jump benchmark when starting at a state with $(n-1)$ one-bits in Theorem~\ref{thmformula}. From this, we easily obtain our lower bound for the runtime of the \MAHH algorithm on \jump functions with valley width $m = O(\sqrt n)$ in Corollary~\ref{cor:SubLinearValley}. 

\begin{theorem}[Expected Duration of the Last Improvement]
\label{thmformula}
We consider the \MAHH algorithm on $\jump_m$.
Let $T_{n-1}^{+}$ be the time for the \MAHH algorithm to reach the state with $n$ one-bits, given a state with $(n-1)$ one-bits. For any $p > 0$, $n \in \mathbb{N}$, and $m \leq n$, we have 
 \begin{equation} \label{eqn:GeneralFormula}
  \expectation[T_{n-1}^{+}] = p^{n-2m+1} \sum_{k=0}^{n-m-1} p^{-k} \binom{n}{k} + p^{1-n} \sum_{k = n-m}^{n-1} \binom{n}{k} p^{k}.     
 \end{equation}
\end{theorem} 

For our proof of this result, we recall a result of Droste, Jansen and Wegener \cite{DrosteJW00}.

\begin{lemma}[{\cite[Corollary~5]{DrosteJW00}}] \label{lem:droste} 
Consider a Markov process $(X_t)_{t \ge 0}$ with state space $[0..n]$. Assume that for all $t \ge 0$, we have $\Pr[|X_t - X_{t+1}| \le 1] = 1$. Let $p_i^{-}$ and $p_i^{+}$ denote the transition probabilities to reach states $i-1$ and $i+1$ from state~$i$, respectively. Let $T_{i}^{+}$ denote the time to reach state $i+1$ when starting in state~$i$. Then
\begin{equation} \label{eqn:ExpandedRecurrenceFormula}
\expectation[T_{i}^{+}] = \sum_{k=0}^{i} \frac{1}{p_k^{+}} \prod_{\ell=k+1}^{i} \frac{p_\ell^{-}}{p_\ell^{+}}.
\end{equation}
\end{lemma}

\begin{proof}[Proof of Theorem~\ref{thmformula}]
Consider a run of the MAHH on $\jump_m$. For all $t \ge 0$, let $X_t$ denote the number of ones of the current solution. Then $(X_t)_{t \ge 0}$ is a Markov process satisfying the assumptions of Lemma~\ref{lem:droste}; note that here we exploit that the MAHH generates the new solution by flipping a single bit. Taking $i = n-1$ in \eqnref{eqn:ExpandedRecurrenceFormula}, we obtain 
\begin{equation}\label{eqn:ExpectedTimeLemma}
\expectation[T_{n-1}^{+}] = \sum_{k=0}^{n-1} \frac{1}{p_k^{+}} \prod_{\ell=k+1}^{n-1} \frac{p_\ell^{-}}{p_\ell^{+}}.
\end{equation}
We will now explicitly compute the $\frac{p_\ell^{-}}{p_\ell^{+}}$ terms. A straightforward calculation counting the number of bits that we can flip, and adding a factor $p$ in cases where the $\jump_m$ value reduces, leads to
\begin{equation}\label{eqn:PossibleFractionValues}
\frac{p_k^{-}}{p_k^{+}} = \left\{
    \begin{array}{ll}
        \frac{k}{n-k} p, & \mbox{if } 0 \leq k \leq n - m - 1; \\
        \frac{k}{n-k} = \frac{n-m}{m}, & \mbox{if } k = n - m; \\
        \frac{k}{n-k} \frac{1}{p}, & \mbox{if } n - m  + 1 \leq k \leq n - 2; \\
        n-1, & \mbox{if } k = n - 1. \\
    \end{array}
\right.    
\end{equation}

We analyze these cases in more detail, beginning with the case where $  0 \leq k \leq n - m - 1$. The product terms in~\eqnref{eqn:ExpectedTimeLemma} can be split as follows, 
\begin{eqnarray}\label{eqn:ExpandedProducts}
\prod_{\ell=k+1}^{n-1} \frac{p_\ell^{-}}{p_\ell^{+}} &=& \left(  \prod_{\ell=k+1}^{n-m-1}\frac{p_\ell^{-}}{p_\ell^{+}} \right) \frac{p_{n-m}^{-}}{p_{n-m}^{+}} \left(\prod_{\ell=n-m+1}^{n-2}   \frac{p_\ell^{-}}{p_\ell^{+}} \right)  \frac{p_{n-1}^{-}}{p_{n-1}^{+}}.
\end{eqnarray}
Plugging the formulas for $\frac{p_k^{-}}{p_k^{+}}$ from \eqnref{eqn:PossibleFractionValues} into  the product terms in \eqnref{eqn:ExpandedProducts} yields
\begin{align*}
     \prod\limits_{\ell=k+1}^{n-m-1} \frac{p_\ell^{-}}{p_\ell^{+}} &= p^{n-m-k-1} \prod\limits_{\ell=k+1}^{n-m-1} \frac{\ell}{n-\ell}, \\
      \prod\limits_{\ell=n-m+1}^{n-2} \frac{p_\ell^{-}}{p_\ell^{+}} &= \left( \frac{1}{p} \right)^{m-2} \prod\limits_{\ell=n-m+1}^{n-2} \frac{\ell}{n-\ell}.    
\end{align*}

Therefore, for  $  0 \leq k \leq n - m - 1$ the summands in \eqnref{eqn:ExpectedTimeLemma} are
\begin{align}
    \frac{1}{p_k^+} \prod_{\ell=k+1}^{n-1} \frac{p_\ell^{-}}{p_\ell^{+}} =& \frac{1}{p_k^+}\left(  \prod_{\ell=k+1}^{n-m-1}\frac{p_\ell^{-}}{p_\ell^{+}} \right) \frac{p_{n-m}^{-}}{p_{n-m}^{+}} \left(\prod_{\ell=n-m+1}^{n-2}   \frac{p_\ell^{-}}{p_\ell^{+}} \right)  \frac{p_{n-1}^{-}}{p_{n-1}^{+}} \notag \\
=& \frac{n}{n-k} p^{n-m-k-1}  \left( \frac{1}{p} \right)^{m-2} \prod_{\ell=k+1}^{n-1} \frac{\ell}{n-\ell} \notag\\
=& p^{n-2m-k+1} \frac{n!}{(n-k)!k!} \notag \\
=& \binom{n}{k}p^{n-2m-k+1}.\label{eqn:FormulaAscent}
\end{align}

In the case where $  n - m  \leq k \leq n - 2$, by once again plugging the formulas for $\frac{p_k^{-}}{p_k^{+}}$ from \eqnref{eqn:PossibleFractionValues} into the product terms in \eqnref{eqn:ExpectedTimeLemma} we obtain
\begin{eqnarray*}
\prod_{\ell=k+1}^{n-1} \frac{p_\ell^{-}}{p_\ell^{+}} &=& \left( \prod_{\ell=k+1}^{n-2} \frac{p_\ell^{-}}{p_\ell^{+}} \right) \frac{p_{n-1}^{-}}{p_{n-1}^{+}} \\
&=& \left( \frac{1}{p} \right)^{n-k-2} \left(\prod_{\ell=k+1}^{n-2} \frac{\ell}{n-\ell} \right) (n-1) \\
&=& \left( \frac{1}{p} \right)^{n-k-2} \frac{(n-1)!}{k!(n-(k+1))!}. 
\end{eqnarray*}
Therefore, for  $  n - m  \leq k \leq n - 2$, the summands in \eqnref{eqn:ExpectedTimeLemma} are
\begin{eqnarray}
\frac{1}{p_k^+} \prod_{\ell=k+1}^{n-1} \frac{p_\ell^{-}}{p_\ell^{+}} &=& \left( \frac{1}{p} \right)^{n-k-1} \frac{n(n-1)!}{k!(n-k)!} \notag\\
&=& \binom{n}{k} \left( \frac{1}{p} \right)^{n-k-1}.\label{eqn:FormulaDescent}    
\end{eqnarray}

Finally, by inserting \eqnref{eqn:FormulaAscent} and \eqnref{eqn:FormulaDescent} into \eqnref{eqn:ExpectedTimeLemma} we obtain
$$\expectation[T_{n-1}^{+}] =  p^{n-2m+1} \sum_{k=0}^{n-m-1} \binom{n}{k} p^{-k}  + p^{1-n} \sum_{k = n-m}^{n-1} \binom{n}{k} p^{k}.$$
\end{proof}

In Theorem \ref{thmformula}, we proved a general formula for the expectation of the time $T_{n-1}^+$ required to go from state $n-1$ to state~$n$. When the MAHH is started in a general state $i \le n-1$, then its expected runtime is $\expectation[T] = \sum_{j=i}^{n-1} \expectation[T_j^+]$. In particular, $\expectation[T_{n-1}^+]$ is a lower bound for the expected runtime. 

Combining this elementary fact with a suitable estimate for $\expectation[T_{n-1}^+]$, we obtain a good lower bound for the runtime in the case that $m = O(\sqrt n)$ in Corollary \ref{cor:SubLinearValley}. We note that larger values of $m$ are not too interesting due to the enormous runtimes (a simple domination argument shows that the expected runtimes are non-decreasing in~$m$). 

\begin{corollary}\label{cor:SubLinearValley}
If $m=O(\sqrt{n})$, then the expected runtime of the \MAHH on $\jump_m$, denoted by $T$, is
$$\expectation[T] = \Omega  \left( \frac{n^{2m-1}}{(2m-1)!}\right).$$
\end{corollary}

\begin{proof}
By the above consideration, noting that the random initial solution has at most $n-1$ one-bits with probability $1 - 2^{-n}$, we have $\expectation[T] \ge (1 - 2^{-n}) \expectation[T_{n-1}^+]$. Considering only the term for $k=n-2m+1$ in \eqnref{eqn:GeneralFormula}, we obtain
\begin{eqnarray}
\expectation\left[T_{n-1}^{+}\right]  & \geq & \binom{n}{n-2m+1}\notag\\
&=& \frac{n(n-1)\cdots(n-2m+2)}{(2m-1)!}\notag\\
&=& \frac{n^{2m-1}(1-\frac{1}{n})\cdots(1-\frac{2m-2}{n})}{(2m-1)!}\notag \\
& \geq &  \frac{n^{2m-1}(1-\frac{2m-2}{n})^{2m-1}}{(2m-1)!}\notag\\ 
&=&\Omega\left(\frac{n^{2m-1}}{(2m-1)!}\right),\notag 
\end{eqnarray}
where the last line relies on the assumption $m = O(\sqrt n)$.
\end{proof}

Consequently, we have shown that the \MAHH has a much larger expected runtime than simple elitist EAs that have a runtime of $O(n^m)$ on jump functions. In Section \ref{sec:GlobalMutation} we will now demonstrate a more positive result.

\section{Upper Bound on the Runtime of the \MAHH with Global Mutation} \label{sec:GlobalMutation}

In Section~\ref{sec:lower_bound} we showed that the \MAHH has a highly noncompetitive runtime on the \jump benchmark, i.e., a runtime of $\Omega  ( \tfrac{n^{2m-1}}{(2m-1)!})$ for gap size $m=O(\sqrt{n})$, significantly larger than  the $O(n^m)$ runtime of many EAs. In light of this performance gap, we propose to use the \MAHH with bit-wise mutation, the mutation operator predominantly used in evolutionary computation, instead of one-bit flips. As we will prove now, this can significantly reduce the runtime of the MAHH on \jump functions. 

We shall observe that the runtime of the MAHH is asymptotically never worse than the runtime of the \oea. It may appear natural that an algorithm equipped with two mechanisms to leave local optima shows such a best-of-two-worlds behavior. However, since the global mutation operator also gives rise to a multitude of new search trajectories not leading to the optimum efficiently, it is not immediately obvious that a combination of two operators results in the minimum performance of the two individual performances. 

For all $m \ge \log_2 n$, our runtime guarantee for the MAHH is asymptotically lower than the runtime of the \oea, roughly by a factor of $\left(O(\frac{\log n}{m})\right)^m$. As our proof suggests, this runtime improvement stems from the fact that the \MAHH is able to {traverse the valley of lower objective values in several steps}. To the best of our knowledge, this is the first time that such an advantage from non-elitism is witnessed for the \jump benchmark.

We note that the runtime guarantees we prove do not beat the best known runtimes. For example, the use of fast mutation~\cite{DoerrLMN17} or stagnation detection~\cite{RajabiW22} gives speed-ups of order $m^{-\Theta(m)}$ over the \oea, with crossover speed-ups of $O(n^{-1})$ or more can be obtained~\cite{DangFKKLOSS18,DoerrEJK24}. Nevertheless, we believe that our result is an interesting demonstration of how non-elitism can be successful in randomized search heuristics, which might aid the future development of non-elitist approaches.

\begin{algorithm}[t]
\caption{\MAHH with global mutation.} \label{alg:global}
\textbf{Input:} Choose $x \in \{0,1 \}^n$ uniformly at random\\
\While{$\mathrm{termination~criterion~not~satisfied}$}{
  $x^{\prime} \gets\text{flip each bit of } x \text{ with probability } \frac{1}{n}$\;
  $\ACC \gets 
  \begin{cases}
      \ALLMOVES, & \text{with probability } p; \\
      \ONLYIMPROVING, & \text{otherwise};\\
  \end{cases}$\\
  \If{$\ACC(x,x^{\prime})$}{
    $x \gets x^{\prime}$\;
  }
}
\end{algorithm}

Precisely, we show the following runtime guarantee for the \MAHH with global mutation (bit-wise mutation with mutation rate~$\frac 1n$, see Algorithm~\ref{alg:global} for the pseudocode).

\begin{theorem}\label{thm:combination}\label{thm:global}
Let $m \in [2..\frac n2]$, let $\gamma \le \frac 14$, and $p \le \gamma \frac{m}{en}$. Then, for all $k \in [1..m]$, the runtime $T$ of the \MAHH with global mutation on $\jump_m$ satisfies 
\begin{equation} \label{eqn:GlobalBound}
\expectation\left[T\right] = O\left(\frac{e^k}{p^{k-1} k^m} \, n^m\right).    
\end{equation}
From this bound, the following conclusions can be drawn. 
\begin{enumerate}
    \item The expected runtime $\expectation[T]$ is $O(n^m)$, which is the runtime of the \oea on $\jump_m$. If $p =  o(\frac mn)$, then $\expectation[T] \le (1+o(1)) e n^m$, that is, apart from lower order terms, not larger than the runtime of the \oea.
    \item  
    We have $\expectation[T] = O(\frac{1}{p} 2^{-m} n^m)$, which is $o(n^m)$ when $m \ge \log_2 n$ and $p = \omega(\frac 1n)$.
    \item We have 
  \[
  \expectation[T] = O\left(\left(\frac{e \ln(e/p)}{m}\right)^m n^m\right),
  \]
which is stronger than the previous bound for $m$ sufficiently large.
\end{enumerate}
In particular, the common choice $p = \Theta(1/n)$ yields a runtime of $\min\{1, O(\frac{e \ln(n)}{m})^m\} O(n^m)$.
\end{theorem}

We proceed now by providing intuition on Theorem \ref{thm:combination} and its proof, to then prove several Lemmas, which will allow us to conclude this section with the formal proof of Theorem \ref{thm:combination}.

To prove Theorem \ref{thm:combination}, we heavily rely on methods developed in the last twenty years in the analysis of randomized search heuristics, in particular, additive and multiplicative drift analysis~\cite{HeY01,DoerrJW12algo} and Wald's equation~\cite{Wald44,DoerrK15}. Due to the more complex search trajectories possible with bit-wise mutation, elementary arguments like the ones used in Section~\ref{sec:lower_bound} appear to be insufficient now. 

To exemplify the differences between the two processes, let us discuss the situation in which the current solution $x$ is in the valley of lower objective values, say adjacent to the local optimum (that is, $\|x\|_1 = n-m+1$). For the classic \MAHH, we now have a high probability (of $\frac{n-m+1}{n}$) of returning to the local optimum. With global mutation, it can happen (and does so with constant probability) that the algorithm flips several one-bits to zero and thereby reaches a solution on the first slope, some steps below the local optimum. Note that such a move is always accepted due to the low objective value of the solution in the valley of lower objective values. From this point on, the algorithm has to then climb up to the local optimum again, which is somewhat difficult (the probability of moving towards the local optimum here is $O(\frac mn)$). In summary, we observe that the time to return to the local optimum (assuming that we do not reach the global optimum first) is $O(1)$ with one-bit mutation and $\Omega(\frac nm)$ with bit-wise mutation. This example also shows why it is not obvious that a combination of two mechanisms to leave local optima gives a best-of-two-worlds behavior. 

\emph{Main proof idea:} To overcome these difficulties, we argue as follows. We split the run time into two terms $T=T_1+T_2$. We first use classic drift arguments to estimate the time $T_1$ required to reach the local optimum (Lemmas \ref{lemma:globaldriftleft}, \ref{lemma:globaldriftright} and \ref{lem:globalt1}). For this, naturally, it is important that the probability of accepting any solution (\allmoves) is low, more precisely, that $p \le c\frac mn$ for a suitable constant $c$.

Once on the local optimum, we use a phase argument based on Wald’s equation to estimate the time $T_2$ to reach the global optimum. Here a phase is the time from starting at a local optimum to reaching it again or reaching the global optimum (Lemmas \ref{lem:globalphaselength} and \ref{lem:globalN}). Since the typical phase consists of just staying at the local optimum, the expected length of such a phase is constant (or even $1 + o(1)$ when $p = o(\frac mn)$); to show this, we need to deal with the situation discussed in the provided example, namely that we jump from the valley of lower objective values onto the slope towards the local optimum, and this needs an additive-drift argument. We note that the expected phase length would also be constant when using one-bit mutation, but more elementary proofs would suffice to show this (as we shall see in Section \ref{sec:upper_bound}). 

The crucial part of the analysis, and where the key difference to using one-bit mutation is found, is estimating the probability that a phase ends in the global optimum. Here we observe that when $m$ is sufficiently large, the typical way the global optimum is reached is not by flipping the $m$ missing bits when on the local optimum, but by traversing the valley of lower objective values from the local optimum in several steps. From the probability that a phase ends in the global optimum, we immediately obtain the expected number of phases, which together with the expected phase length gives the expected time to reach the global optimum (via Wald's equation).

To make this precise, let $(X_t)_{t\geq0}$ denote the sequence of states of a run of the \MAHH on $\jump_m$. Let $X^*$ denote the set of local and global optima, that is, 
\[
X^* = \{x \in \{0,1\}^n \mid \|x\|_1 \in \{n-m,n\}\}.
\]
We furthermore denote the global optimum by $x^*$, i.e., $\|x^*\|_1=n$. 
Let 
\[T_1 = \min\{t \in \N_0 \mid X_t \in X^*\}
\]
denote the time to reach a local optimum or the global optimum (we remark without proof that almost always, a local optimum will be reached first). 

In our drift arguments, we will use the following potential function, measuring the distance to the local optimum unless the global optimum is given as input. 
\[
d(x) = 
\begin{cases}
    \left|n - m - \Vert x \Vert_1\right|, &\text{if } x \neq x^*; \\
    0, & \text{otherwise.}
  \end{cases}
\]

Both in our analysis of the time to reach the local optimum and in the remaining run, we need a good understanding of the drift with respect to this progress measure. This is what we obtain in the following Lemmas \ref{lemma:globaldriftleft} and \ref{lemma:globaldriftright}, discussing separately the case that the current solution is on the \onemax-style slope towards the local optima, i.e., at a position $x$ such that $\lVert x\rVert_1<n-m$, and the case that it is in the gap region, i.e., $n-m < \lVert x\rVert_1 < n$.

\begin{lemma}[Drift in the slope towards the local optimum] \label{lemma:globaldriftleft}
Let $\gamma \ge 0$ be a constant. Let $p \le \gamma \frac{m}{en}$. Let $x \in \{0,1\}^n$ with $\|x\|_1 < n-m$. Then
\[
\expectation\left[d\left(X_t\right)-d\left(X_{t+1}\right)\mid X_t=x\right] \geq (1-p)\frac{d(x)}{en} + (1-p-\gamma) 
\frac{m}{en}.
\]
\end{lemma}

\begin{proof}
  Since $\|x\|_1 < n-m$, the search point $x$ is not in the gap region of the $\jump$ function, but in the \onemax-type slope towards the local and global optima. In particular, $x$ has $d(x)+m$ zero-bits and flipping a single one of them and no other bit results in a reduction of the $d$-value by one. 

In an iteration where the \onlyimproving operator is used, only objective value improvements and thus only $d$-reductions will be accepted. In other words, since the $d$-value cannot increase, we can estimate the drift simply by regarding the progress made from a single desirable event. We regard the event that a single zero-bit is flipped into a one-bit. This event has probability $(d(x)+m)\frac 1n (1-\frac 1n)^{n-1} \ge \frac{d(x)+m}{en}$. Since it improves the $d$-value by one, the expected progress $d(X_t)-d(X_{t+1})$ is also at least $\frac{d(x)+m}{en}$ in this case.

  In an iteration where the \allmoves operator is used, we can pessimistically estimate the increase of the $d$-value by the number of bits that are flipped, which in expectation is one.

  In summary, we obtain
  \begin{align*}
    \expectation\left[d\left(X_t\right)-d\left(X_{t+1}\right)\mid X_t=x\right] 
    &\geq (1-p)\frac{d(x)+m}{en} - p \cdot 1\\
    &\geq (1-p)\frac{d(x)}{en} + (1-p-\gamma) \frac{m}{en}.\qedhere
  \end{align*}

\end{proof}

\begin{lemma}[Drift in the gap region]
\label{lemma:globaldriftright}
Let $\gamma \ge 0$ be a constant. Let $p \le \gamma \frac{m}{en}$. Let $x \in \{0,1\}^n$ with $n-m < \|x\|_1 < n$. Then 
\[
\expectation\left[d\left(X_t\right)-d\left(X_{t+1}\right)\mid X_t=x\right] \geq \frac{1-17p}{16}.
\]
\end{lemma}

\begin{proof}
Ignoring the event that we jump right into the optimum (which reduces the $d$-value), in the case of an \onlyimproving iteration, the $d$-value can only change if at least one $1$-bit is flipped to zero. 

Let us, by symmetry, assume that $x_1 = \cdots = x_{n-m+d} = 1$ and $x_{n-m+d+1} = \cdots = x_n = 0$. For $i \in [1..(n-m+d)]$, let $A_i$ be the event that the mutation operator does not flip any of the first $i-1$ bit, but does flip the $i$-th bit. In other words, the $i$-th bit is the first (leftmost) bit to flip. By our initial consideration, in an \onlyimproving iteration starting in state $x$, we have
\begin{multline}
  \expectation\left[d\left(X_t\right)-d\left(X_{t+1}\right)\mid X_t=x\right] \ge\\ \sum_{i=1}^{n-m+d} \Pr[A_i \mid X_t=x] \, \expectation\left[d\left(X_t\right)-d\left(X_{t+1}\right) \mid A_i, X_t=x\right].    
\end{multline}
We easily see that $\Pr[A_i \mid X_t=x] = (1-\frac 1n)^{i-1} \frac 1n$. To estimate $\expectation\left[d\left(X_t\right)-d\left(X_{t+1}\right) \mid A_i\right]$, assume now that $A_i$ holds. If the $i$-th bit was the only one to be flipped, we would decrease the $d$-value by exactly one. Any additional bit that is flipped can change this progress by at most one. Hence, since the expected number of such additional flipped bits is $\frac{n-i}{n}$, we have
\[
  \expectation\left[d\left(X_t\right)-d\left(X_{t+1}\right) \mid A_i, X_t=x\right]
  \ge 1 - \frac{n-i}{n}.
\]
Consequently, still assuming an \onlyimproving iteration, we compute
\begin{align*}
  \expectation\left[d\left(X_t\right)-d\left(X_{t+1}\right) \mid X_t=x\right] 
  &\ge \sum_{i=1}^{n-m+d} \left(1-\frac 1n\right)^{i-1} \frac 1n  \left(1 - \frac{n-i}{n}\right)\\
  &\ge \frac 1 {n^2} \sum_{i=1}^{n-m+d} \left(1-\frac {i-1}n\right) i\\
  &\ge \frac 1 {n^2} \sum_{i=1}^{\lceil n/2 \rceil} \tfrac 12 i \ge \frac 1 {16}.
\end{align*}

For an iteration in which the \allmoves operator was selected, we can bluntly estimate the change of $d$ by the number of bits that were flipped, which in expectation is one. Consequently, looking at both types of iterations together, we have
\begin{align*}
  \expectation\left[d\left(X_t\right)-d\left(X_{t+1}\right)\mid X_t=x\right]
  &\geq (1-p) \frac 1{16} - p \cdot 1 = \frac{1-17p}{16}.\qedhere
\end{align*}
\end{proof}

From these drift estimates, we first derive an estimate for the expected time to reach the local or global optimum. The main proof idea is to use multiplicative drift when the current solution is more than $m$ away from the local optimum and then additive drift.
\begin{lemma}\label{lem:globalt1}
  Let $m \le \frac n2$. Let $\gamma = \frac 14$ and $p \le \gamma \frac{m}{en}$. Then the expected time to reach a local or the global optimum is
  \[
  \expectation[T_1] = O(n \log \tfrac{n}{m}).
  \]
\end{lemma}

\begin{proof}
  We first use multiplicative drift to argue that $O(n \log \frac{n}{m})$ iterations suffice to reach a solution $x$ with $d(x) \le m$. To this end, we consider the modified potential $d'$ defined by $d'(x) = d(x)$ if $d(x) \ge m$ and $d'(x) = 0$ otherwise. By Lemma~\ref{lemma:globaldriftleft}, we have 
  \[
  \expectation[d'(X_{t+1}) \mid d'(X_t) = d'(x)] \le \left(1 - \frac{(1-p)}{en}\right) d'(x)
  \]
  for all $x \in \{0,1\}^n$ with $d'(x)>0$ (note that we do not need Lemma~\ref{lemma:globaldriftright} here since all search points covered by it have a $d'$-value of zero). Consequently, by the multiplicative drift theorem (Theorem~\ref{thm:mult-drift}) the expected time to reach a solution with $d'(x)=0$, that is, $d(x)<m$, is at most 
  \[
  \frac{1 + \ln \frac nm}{\frac{(1-p)}{en}} = O\left(n \log \frac nm\right).
  \]

  Once we have reached a solution with $d(x) \le m$, we use additive drift. By choice of $\gamma$ and using $m \le \frac n2$, we have that both $(1-p-\gamma)$ and $\frac{1-17p}{16}$ are bounded below from a positive constant $\Gamma$. Hence by Lemmas~\ref{lemma:globaldriftleft} and~\ref{lemma:globaldriftright}, we have \begin{equation}\label{eq:adddrift}
    \expectation[d(X_t) - d(X_{t+1}) \mid d(X_t) =x] \ge \Gamma \, \frac{m}{en}  
  \end{equation} 
  for all $x \in \{0,1\}^n$ such that $d(x) > 0$. Consequently, by the additive drift theorem (Theorem~\ref{thm:adddrift}) the expected time to reach a $d$-value of zero is at most $m / (\Gamma \frac{m}{en}) = O(n)$. 

  In summary, the expected time to reach a local or global optimum is 
  \[
  \expectation[T_1] = O\left(n \log \frac nm\right) + O(n) = O\left(n \log \frac nm\right).\qedhere
  \]
\end{proof}

To estimate the remaining runtime, we consider the time $T_2$ to reach the global optimum when starting in a local optimum. We define the sequence of random random times $(P_i)_{i \geq 0}$ when the current solution is a local or global optimum by
\begin{align}
\begin{split}\label{eqn:PartitionDefintion}
        P_0 &= 0, \\
        P_{i+1} &= \min\{t > P_i \mid X_t \in X^* \}. \end{split}   
\end{align}
Let $N = \min\{i \mid X_{(P_i)} = x^*\}$ be the number of such phases until the global optimum is reached. Then
\[
T_2 = P_N = \sum_{i = 1}^{N} (P_{i}-P_{i-1}). \] 
Since $(X_t)_{t\geq0}$  satisfies the Markov property, the phase lengths $\left(P_{i} - P_{i-1}\right)_{i \geq 1}$ are independent and identically distributed. 
To estimate $T_2$ via Wald's equation, we prove upper bounds on the length of such a phase and the number~$N$ of these phases. 

We first show that under the assumptions made in the previous lemmas, the expected phase length is constant. It is $1+o(1)$ when $p = o(m/n)$. The key ingredients of the following proof are our additive drift estimates in \eqnref{eq:adddrift} and the observation that an \allmoves iteration increases the $d$-value of the solution by at most one in expectation.  

\begin{lemma}
\label{lem:globalphaselength}
Let $m \le \frac n2$. Let $\gamma = \frac 14$ and $p \le \gamma \frac{m}{en}$. Then for each phase $i \in [1..N]$, we have
\[
\expectation\left[P_{i}-P_{i-1}\right] = 1 + O\left(\frac{pn}{m}\right).
\]
\end{lemma}

\begin{proof}
  Clearly, if in the first iteration of this phase the \onlyimproving operator is used, then the length of the phase is one (the only accepted solutions are other local optima or the global optimum). 

  Consider now the case that the first iteration of this phase employs the \allmoves operator. Now any solution is accepted. We analyze the remaining time from this point on. The new solution $x'$ has an expected Hamming distance of one from the parent. Consequently, $\expectation[d(x')] \le 1$. As in the proof of Lemma~\ref{lem:globalt1}, in~\eqref{eq:adddrift}, we now have an additive drift in $d$ of order $\Omega(\frac mn)$. Hence, after an expected number of $O(d(x') / \frac mn)$ iterations, we have reached again a solution with $d$-value zero. By the law of total expectation (applied to the outcomes of $x'$), after an expected number of $O(\expectation[d(x)] / \frac mn) = O(\frac nm)$ additional iterations, this phase ends. 

  Taking into account the two cases, we see that
  \[
  \expectation[P_i - P_{i-1}] = (1-p) \cdot 1 + p  \cdot (1+O(\tfrac nm)) = 1 + O\left(\frac{pn}{m}\right).\qedhere
  \]

\end{proof}

We now turn to the analysis of the expected number of phases. Invoking the Markov property again, this is the reciprocal of the probability that a phase ends in the global optimum. Estimating this probability, that is, proving strong lower bounds on it, is the key to showing upper bounds on the runtime which can beat the runtime of the \oea. The central observation, and reason for the progress over the preliminary version~\cite{DoerrDLS23} of this work, is that a substantial contribution to this probability may stem from phases in which the MAHH accepts an inferior solution in the valley of lower objective values several times. This observation builds upon a precise computation (and subsequent estimation) of the probability that the phase describes a walk of length $k$ through the valley of lower objective values in which only $0$-bits are flipped (and consequently, the distance to the optimum never increases). 

\begin{lemma}\label{lem:globalN}
  Let $p \le (1-\frac 1n)$. Then for all $k \in [1..m]$, we have 
  \[
  \expectation[N] \le \frac{e^k}{p^{k-1}(k^m - (k-1)^m)} n^m \le \frac{e^2}{e-1} \left(\frac{e}{p}\right)^{k-1} k^{-m} \, n^m.
  \]
\end{lemma}

\begin{proof}
  Since we have a Markov process and since each phase starts in an equivalent situation, the  number of phases follows a geometric law with success rate equal to the probability that a phase ends in the global optimum. In particular, the expectation of $N$ is the reciprocal of this probability. We thus show a lower bound for this probability now.

  Let $k \in [1..m]$ and $m_1, \dots, m_k \in [0..m]$ such that $m = m_1 + \dots + m_k$. Assume that the current solution of the \MAHH at some time $t$ is a local optimum. Let $A_{m_1, \dots, m_k}$ be the event that for all $i \in [0..k]$ the solutions $X_{t+i}$ at time $t+i$ have Hamming distance $H(x,x^*) = m - \sum_{j=1}^i m_i$ from the optimum, and that each new solution was generated from the previous one via a mutation that does not flip ones into zeroes. Then 
  \[
  \Pr[A_{m_1, \dots, m_k}] = \frac 1 p \prod_{i=1}^k \left(p^{\mathbf{1}_{m_i>0}} \binom{m - \sum_{j=1}^{i-1} m_j}{m_i} \left(1-\frac 1n\right)^{n-m_i} \frac{1}{n^{m_i}}\right).
  \]
  By estimating $(1-\frac 1n)^{n - m_i} p^{\mathbf{1}_{m_i>0}} \ge \frac 1e p$, we obtain
  \[
  \Pr[A_{m_1, \dots, m_k}] \ge p^{k-1} e^{-k} n^{-m} \frac{m!}{m_1! \cdots m_k!}.
  \]
  Let $M$ be the set of all $(m_1, \dots, m_k) \in [0..m]^k$ such that $m_1 + \dots + m_k = m$ and $m_1 \ge 1$. Let $A = \bigcup_{(m_1, \dots, m_k)\in M} A_{m_1, \dots, m_k}$. Note that $A$ implies that this phase ends with the optimum found. We have
  \[
  \Pr[A] = \sum_{(m_1, \dots, m_k) \in M} \Pr[A_{m_1, \dots, m_k}] \ge p^{k-1} e^{-k} n^{-m} \sum_{(m_1, \dots, m_k) \in M}  \frac{m!}{m_1! \cdots m_k!}.
  \]
  We note that $\frac{m!}{m_1! \cdots m_k!}$ is the number of ways an $m$-element set $S$ can be partitioned into sets $S_1, \dots, S_k$ such that $|S_i|=m_i$ for all $i \in [1..k]$. Consequently, $\sum_{(m_1, \dots, m_k) \in M}  \frac{m!}{m_1! \cdots m_k!}$ is the number of ways such a set can be partitioned into sets $S_1, \dots, S_k$ of arbitrary cardinalities except that $S_1$ is not empty. The number of ways to partition an $m$-set into $k$ arbitrary sets is equal to~$k^m$, since for each of the $m$ elements we have $k$ sets to choose from. Consequently, the number of such partitions where the first set is non-empty is $k^m - (k-1)^m$. In conclusion, we have
  \[
  \Pr[A] \ge p^{k-1} e^{-k} (k^m - (k-1)^m) n^{-m}.
  \]
  This lower bound on the probability to find the optimum in one phase immediately gives the following upper bound on the expected number of phases, namely
  \[
  \expectation[N] \le \frac{1}{\Pr[A]} \le \frac{e^k}{p^{k-1}(k^m - (k-1)^m)} n^m.
  \]
  We can estimate this expression by computing $k^m - (k-1)^m = k^m (1 - (1-\frac 1k)^m \ge k^m (1 - (1-\frac 1m)^m \ge k^m(1-\frac 1e)$, giving
  \[
  \expectation[N] \le \frac{e^{k+1}}{(e-1) p^{k-1}k^m} n^m. \qedhere
  \]
\end{proof}

We are now ready to prove Theorem~\ref{thm:global}, the main result of this section.

\begin{proof}[Proof of Theorem~\ref{thm:global}]
  By Lemma~\ref{lem:globalt1}, it takes an expected number of $O(n \log \frac nm)$ iteration to reach a local optimum. 

  From that point on, we study the times $P_i$, $i \in [0..N]$, defined in \eqnref{eqn:PartitionDefintion}. The remaining runtime $T_2$ is the sum of the phase lengths $P_{i}-P_{i-1}$, that is, 
  \[
  T_2 = \sum_{i=1}^N (P_i - P_{i-1}).
  \]
  By Lemma~\ref{lem:globalphaselength}, the phase lengths $P_i - P_{i-1}$ have an expectation uniformly bounded some absolute value~$C$. This allows to use a basic version of Wald's equation, Theorem~\ref{thm:wald}, to conclude that 
  \begin{equation}\label{eq:waldappli}
  \expectation[T_2] \le \expectation[N] C = O(\expectation[N]).     
  \end{equation}
  Hence Lemma~\ref{lem:globalN} proves our general result. Here we note that our (first) upper bound on $\expectation[N]$ satisfies
  \[
  \frac{e^k}{p^{k-1}(k^m - (k-1)^m)} n^m \ge \frac{n^m}{k^m} \ge \left(\frac{n}{m}\right)^m, 
  \]
  hence the $O(n \log \frac nm)$ term is subsumed in our estimate for $\expectation[N]$ for all $m \in [2..\frac n2]$. 
  
  We now prove the particular estimates. By taking $k=1$, we see that $\expectation[N] \le en^m$ and thus $\expectation[T]=O(n^m)$. When $p = o(\frac mn)$, the number $C$ in~\eqref{eq:waldappli} can be chosen as $1+o(1)$. Hence $\expectation[T] \le O(n \log \frac nm)+(1+o(1))\expectation[N] = (1+o(1)) e n^m$. 

  Regarding the general bound for $k=2$, we obtain $\expectation[T] = O(\frac{1}{p (2^m-1)}n^m)$, which is $o(n^m)$ when $p = \omega(\frac 1n)$ and $m \ge \log_2(n)$.

  By taking $k = \lceil \frac{m}{\ln(e/p)} \rceil$ and estimating $(e/p)^k \le (e/p)^{\frac{m}{\ln(e/p)}} \cdot (e/p) = e^m (e/p)$, we obtain
  \[
  \expectation[N] = O\left(\left(\frac{e \ln(e/p)}{m}\right)^m n^m\right).\qedhere
  \]
\end{proof}

\section{Upper Bound on the Runtime of the \MAHH with One-Bit Mutation}
\label{sec:upper_bound}

We will now derive an upper bound on the runtime of the \MAHH on $\jump_m$ in the case $p =\frac{m}{n}$. It will, in particular, show that our lower bound in Section~\ref{sec:lower_bound} is tight apart from factors depending on $m$ only. 

Following the outline of the proof for the more complex bit-wise mutation operator in Section \ref{sec:GlobalMutation}, the proof in this section is comparatively easy to establish. We begin by stating the main result of this section. 

\begin{theorem}\label{thm:UpperBound}
Let $m \in [2..\frac n2]$ and $p=\frac{m}{n}$. Then the runtime $T$ of the \MAHH on $\jump_m$ satisfies 
$$\expectation\left[T\right] = O\left(
\frac{n^{2m-1}}{m!m^{m-2}}\right).$$
\end{theorem}

The main proof idea is the same as for Theorem \ref{thm:global} in Section \ref{sec:GlobalMutation}, i.e., in order to prove Theorem \ref{thm:UpperBound}, we split the expected runtime $T$ in two times,
$$T = T_1 + T_2,$$
where $T_1$ is the time to reach a local or the global maximum from an arbitrary initial solution (hence also a random initial solution)  
and $T_2$ is the additional time to reach the global optimum. 

We  resort to the same potential function $d$ as in Section \ref{sec:GlobalMutation}, which we recall now,
\[
d(x) = 
\begin{cases}
    \left|n - m - \Vert x \Vert_1\right|, &\text{if } x \neq x^*; \\
    0, & \text{otherwise.}
  \end{cases}
\]

In the now following Lemmas \ref{lemma:drift} and \ref{lemma:drift2} we will, respectively, bound from below the drift towards the local optimum when we are firstly, on the \onemax-style slope towards the local optimum and secondly, in the gap region of lower objective values. 

\begin{lemma}[Drift in the slope towards the local optimum]
\label{lemma:drift}
Let $p \leq \frac{m}{n}$ and $x\in\{0,1\}^n$ such that $\Vert x \Vert_1 < n-m$. Then 
$$\expectation\left[d\left(X_t\right)-d\left(X_{t+1}\right) \,\middle|\, X_t=x\right] \geq \frac{d(x)}{n}. $$
\end{lemma}

\begin{proof}
The required expectation can straightforwardly be computed by regarding separately the two types of acceptance operators and noting that in each case, the potential function $d$ can change by at most one. Noting that $\|x\|_1 = n-m-d(x)$ and $p \le \frac mn$, we compute
\begin{align*}
\expectation[d&(X_t) - d(X_{t+1}) \mid X_t = x] 
= p \left(\frac{n-\|x\|_1}{n}- \frac{\|x\|_1}{n}\right) + (1-p) \frac{n-\|x\|_1}{n}\\
& = \frac{n-\|x\|_1}{n} - p \frac{\|x\|_1}{n} = \frac{m+d(x)}{n} - p \frac{\|x\|_1}{n} \ge \frac {d(x)}n. & \qedhere   
\end{align*}
\end{proof}

\begin{lemma}[Drift in the gap region] \label{lemma:drift2}
Let $m \le \frac n2$. Let $x\in\{0,1\}^n$ such that $\Vert x \Vert_1 > n-m$. Then
$$\expectation\left[d\left(X_t\right)-d\left(X_{t+1}\right) \,\middle|\, X_t=x\right] \geq \frac{d(x)}{n}. $$
\end{lemma}

\begin{proof}
As in the proof of Lemma \ref{lemma:drift} we can straightforwardly compute the required expectation, noting that now $\|x\|_1 = n - m + d(x)$ and $d(x) \le m$.
\begin{align*}
\expectation[d&(X_t) - d\left(X_{t+1}\right) \mid X_t=x]  
 = p \left(\frac{\Vert x \Vert_1}{n} - \frac{n - \Vert x \Vert_1}{n}\right) + (1-p) \frac{\Vert x \Vert_1}{n}\\
& = \frac{\Vert x \Vert_1}{n} - p\frac{n - \Vert x \Vert_1}{n} = \frac{n-m+d(x)}{n} - p\frac{m-d(x)}{n}\\
& \ge  \frac{n-m+d(x)}{n} - \frac{m-d(x)}{n} = \frac{n - 2m}{n} + \frac{2d(x)}{n} \ge \frac{2d(x)}{n} \ge \frac {d(x)}n.\nonumber \qedhere
\end{align*}
\end{proof}

The two estimates above will suffice to obtain an upper bound on the expectation of $T_1$. To bound $T_2$, as in Section \ref{sec:GlobalMutation}, we define phases of random times $(P_i)_{i \geq 0}$ when $X$ returns to the set $X^*$ of local and global maxima.
\begin{align*}
        P_0 &= T_1, \\
        P_{i+1} &= \min\{k \text{ : } k \geq P_{i}+1 \text{ and } X_{k} \in X^* \}.    
\end{align*}
Let $N = \min\{i \ge 0 \mid X_{P_i} = x^*\}$. Then
$$T_2 = \sum_{i = 1}^{N} P_{i}-P_{i-1}.$$ 
As the $(X_t)_{t\geq0}$  verifies the Markov property, the $(P_{i} - P_{i-1})_{i \geq 0}$ are independent and follow the same distribution. We first prove an upper bound on the duration of the expected phase length.

\begin{lemma}
\label{lm:P}
Let $p \leq \frac{m}{n}$. Then for each phase $i\in[1..N]$ we have 
$$\expectation\left[P_{i}-P_{i-1}\right] = O(m).$$
\end{lemma}

\begin{proof}
We observe that from the local optimum three possible moves exist. (1)~The MAHH stays at the local optimum. In this case the phase length is one. (2)~The MAHH moves left of the local optimum, that is, $\lVert X_{P_i+1}\rVert_1=n-m-1$. Then $d(X_{P_i+1}) = 1$ and the multiplicative drift theorem (Theorem \ref{thm:mult-drift}) with $s_0 = 1$ and $\delta = \frac 1n$ (from Lemma \ref{lemma:drift}) gives an expected phase length of $O(n)$. (3)~The MAHH moves one step into the gap region. Then again $d(X_{P_i+1}) = 1$ and the multiplicative drift theorem with $s_0 = 1$ and $\delta = \frac 1n$ (now from Lemma \ref{lemma:drift2}) gives an expected phase length of $O(n)$.

The two later cases require the acceptance of an inferior solution, hence they (together) occur with probability $p$. Hence by the law of total expectation, we have
\begin{align*}
\expectation\left[P_{i}-P_{i-1}\right] 
&= p \cdot O(n)  + (1 - p) \cdot 1 = O(m),
\end{align*}
using that $p \leq \frac{m}{n}$.
\end{proof}

The second ingredient to our argument based on Wald's equation is an estimate of the expectation of the number~$N$ of phases.

\begin{lemma}
\label{lm:N}
If $p\geq \frac{m}{n}$, then number $N$ of phases satisfies 
\[\expectation\left[N\right] = O\left(\frac{n^{2m-1}}{m!m^{m-1}}\right).\]
\end{lemma}

\begin{proof}
As in Lemma~\ref{lem:globalN}, the expected number of phases is the reciprocal of the probability that a phase ends in the global optimum. This probability is at least the probability that the phase consists of exactly $m$ iterations, each bringing the current solution by one step closer to the global optimum. Let $p_i$ be the probability that an iteration starting with a solution at distance $i$ from the global optimum ends with one at distance $i-1$. For $i \in [2..m]$, we have $p_i = \frac in p$, since there are $i$ bits that can be flipped to obtain the desired solution and further the new, inferior, solution has to be accepted. For $i=1$, we have $p_i=\frac 1n$ since the desired solution is better than the previous one. 

In summary, and using the assumption $p \ge \frac mn$, we have that a phase ends with the global optimum with probability at least
\[\prod_{i=1}^m p_i =  \frac{m!}{n^m}p^{m-1} \ge \frac{m! m^{m-1}}{n^{2m-1}}.\]
Consequently, $E[N] \le \frac{n^{2m-1}}{m! m^{m-1}}$.
\end{proof}

From the intermediate results above, we now easily prove the main result of this section, the upper bound of Theorem~\ref{thm:UpperBound}. 

\begin{proof}[Proof of Theorem~\ref{thm:UpperBound}]
Recall that we split the runtime $T$ into the sum of two runtimes $T_1$ and $T_2,$ where $T_1= \min \{ t \in \mathbb{N}_0 \mid X_t\in X^* \}$ is the first time to reach a local or global optimum.

From Lemmas \ref{lemma:drift} and \ref{lemma:drift2} together with the multiplicative drift theorem (Theorem \ref{thm:mult-drift}) applied with $s_0 = d(X_{0}) = n-m$, $s_{\min}=1$ and $\delta = \frac{1}{n}$, we obtain
$\expectation\left[T_1\right] = O(n \log n)$.

To estimate the remaining time, that is, the time $T_2$ required to reach the global optimum after the first arrival at a local optimum (where we pessimistically assume that the first phase has not ended on the global optimum), 
we can combine the results of Lemmas~\ref{lm:P} and \ref{lm:N} and Wald's formula (Theorem \ref{thm:wald}), to obtain that if $p=\frac{m}{n},$
\begin{eqnarray*}
\expectation\left[T_2\right] &=& \expectation\left[N\right]\expectation\left[P_1-P_0\right]\\
&=&O(m\expectation\left[N\right])\\
&=& O\left(\frac{n^{2m-1}}{m!m^{m-2}}\right).
\end{eqnarray*}
This proves Theorem~\ref{thm:UpperBound}. 
\end{proof}

\section{Conclusion}\label{sec:conclusion}

In this work, we conducted a precise runtime analysis of the \MAHH on the \jump benchmark, the most prominent multimodal benchmark in the theory of randomized search heuristics. We proved that the \MAHH on the \jump benchmark does not replicate the extremely positive performance it has on the \cliff benchmark~\cite{LissovoiOW23}, but instead has a runtime that for many jump sizes, in particular, the more relevant small ones, is drastically above the runtime of many evolutionary algorithms. This could indicate that the \cliff benchmark is a too optimistic model for local optima in heuristic search. 

On the positive side, we propose to use the \MAHH with the global bit-wise mutation operator common in evolutionary algorithms and prove the non-obvious result that this leads to a runtime which is essentially the minimum of the runtimes of the classic \MAHH and the \oea. Excitingly, this result is proved using an argument, which demonstrates an advantage of the \MAHH traversing the valley of low objective values in several steps. 

Since we observed this best-of-two-worlds effect so far only on the \jump benchmark, more research in this direction is clearly needed. We note that in general it is little understood how evolutionary algorithms and other randomized search heuristics profit from non-elitism. So, also more research on this broader topic would be highly interesting.
		
\subsection*{Acknowledgment}

This research benefited from the support of the FMJH Program Gaspard Monge for optimization and operations research and their interactions with data science.

\bibliographystyle{alpha}
\bibliography{ich_master,alles_ea_master,rest}

	}
\end{document}